\documentclass[runningheads, 12pt]{llncs}
\usepackage{amsmath, amssymb}
\usepackage{graphicx}
\usepackage{algorithm}
\usepackage{algorithmic}
\usepackage{cite}
\usepackage{hyperref}
\usepackage{xcolor}
\usepackage{booktabs, comment,subcaption,enumerate,footmisc,paralist}
\usepackage{geometry}
\title{Convexity-Driven Projection for Point Cloud Dimensionality Reduction}
\titlerunning{Convexity-Driven Projection}
\author{Suman Sanyal\inst{1}}
\authorrunning{S. Sanyal}
\institute{Big Data Analytics, Goa Institute of Management, Goa, India\\
\email{sanyal@gim.ac.in}}

\begin{document}
\maketitle

\begin{abstract}
We propose Convexity-Driven Projection (CDP), a boundary-free linear method for dimensionality reduction of point clouds that targets preserving detour-induced local non-convexity. CDP builds a $k$-NN graph, identifies admissible pairs whose Euclidean-to-shortest-path ratios are below a threshold, and aggregates their normalized directions to form a positive semidefinite non-convexity structure matrix. The projection uses the top-$k$ eigenvectors of the structure matrix. We give two verifiable guarantees. A pairwise a-posteriori certificate that bounds the post-projection distortion for each admissible pair, and an average-case spectral bound that links expected captured direction energy to the spectrum of the structure matrix, yielding quantile statements for typical distortion. Our evaluation protocol reports fixed- and reselected-pairs detour errors and certificate quantiles, enabling practitioners to check guarantees on their data.

\end{abstract}

\section{Introduction}
\label{sec:intro}

Point clouds, sets of points in high-dimensional spaces, are central to 3D modelling, robotics, and visualisation~\cite{b8}. In many of these settings, obstacle- or curvature-induced detours are essential: the graph shortest path between two samples can be much longer than their Euclidean separation, signalling local non-convexity. Preserving this detour geometry matters for tasks such as path planning and shape analysis, yet standard dimensionality-reduction (DR) methods do not prioritise it. Variance-seeking linear methods (PCA) maximise explained variance and can collapse detours by ignoring geodesic structure~\cite{pca}. Nonlinear neighbour-embeddings such as t-SNE~\cite{tsne} and UMAP~\cite{umap} optimise probabilistic/fuzzy neighbourhood objectives; they are effective for visualisation but offer no explicit control of detour geometry and yield nonlinear maps sensitive to hyperparameters. Graph-based linear projections (LPP, NPP, OLPP)~\cite{heNiyogi2004lpp,he2005npe,cai2006olpp} minimise Laplacian-style smoothness to keep neighbours close, which does not single out directions that consistently witness non-convexity.

We propose \textbf{Convexity-Driven Projection} (CDP), a linear DR method designed to preserve detour-induced local non-convexity in point clouds without requiring boundary estimation. CDP constructs a \(k\)-nearest-neighbor (\(k\)-NN) graph, identifies admissible pairs \((i,j)\) with a convexity ratio \(r_{ij} = \|p_j - p_i\| / S(p_i, p_j) \le \tau\), where \(S(p_i, p_j)\) is the graph shortest-path distance, and forms a positive semidefinite (PSD) non-convexity structure matrix by averaging normalized outer products \(u_{ij} u_{ij}^\top\) weighted by \((1 - r_{ij})\). The projection is onto the top-\(k\) eigenvectors of this matrix, emphasizing directions associated with strong detours. CDP’s linear nature ensures computational simplicity, while its boundary-free approach makes it broadly applicable. The contributions of the paper are as follows. (i) A simple PSD construction that targets detour-induced structure; (ii) a pairwise a-posteriori certificate bounding $\tilde r_{ij}/r_{ij}$ using the projected shortest path;
(iii) an average-case spectral bound with quantile implications; and (iv) an evaluation protocol reporting fixed- and reselected-pairs metrics with certificate quantiles.

This paper is organized as follows. Section~\ref{sec:prelim} introduces key notations and the mathematical foundation of CDP. Section~\ref{sec:algo} details the CDP algorithm, including complexity and graph construction. Section~\ref{sec:theory} presents the theoretical guarantees with a verification protocol, including the pairwise certificate and spectral bound. Section~\ref{sec:results} evaluates CDP on synthetic and real-world datasets, including a toy example. Section~\ref{sec:conclusion} summarizes findings and future directions. Table~\ref{tab:notations} lists key notations used throughout the paper, and a comparison of CDP with existing methods is provided in Table~\ref{tab:comparison}.

\begin{table}[t]
\centering
\caption{Key notations and their descriptions used in the paper.}
\label{tab:notations}
\scalebox{0.9}{
\begin{tabular}{c|l}
\toprule
Notation & Description \\
\midrule
$p_i$ & Point \(i\) in the original space \(\mathbb{R}^d\), \(i=1,\dots,N\) \\
$G=(V,E)$ & Connected \(k\)-NN graph with vertices \(V\) and edges \(E\) \\
$w(u,v)$ & Euclidean edge weight for edge \((u,v) \in E\), \(\|p_v - p_u\|\) \\
$S(p_i,p_j)$ & Shortest-path distance between points \(p_i\) and \(p_j\) on \(G\) \\
$r_{ij}$ & Convexity (detour) ratio, \(\frac{\|p_j - p_i\|}{S(p_i,p_j)}\), Eq.~\eqref{eq:rij} \\
$\mathcal{D}^*$ & Set of admissible pairs \((i,j)\) with \(r_{ij} \le \tau\), \(\tau < 1\) \\
$\widehat{C}^{\mathrm{sp}}$ & Admissible non-convexity index, \(\frac{1}{|\mathcal{D}^*|} \sum_{(i,j) \in \mathcal{D}^*} r_{ij}\) \\
$u_{ij}$ & Normalized direction, \(\frac{p_j - p_i}{\|p_j - p_i\|}\) \\
$S^{\mathrm{nc}}$ & Non-convexity structure matrix, Eq.~\eqref{eq:Snc} \\
$\lambda_\ell$, $z_\ell$ & Eigenvalues and eigenvectors of \(S^{\mathrm{nc}}\), \(\lambda_1 \ge \dots \ge \lambda_d \ge 0\) \\
$V$ & Projection matrix, top-\(k\) eigenvectors \([z_1, \dots, z_k]\) \\
$p_i'$ & Projected point, \(V^\top p_i \in \mathbb{R}^k\) \\
$w'(u,v)$ & Projected edge weight, \(\|V^\top (p_v - p_u)\|\) \\
$\tilde{S}(p_i,p_j)$ & Shortest-path distance on projected graph \((V,E,w')\) \\
$\tilde{r}_{ij}$ & Post-projection convexity ratio, \(\frac{\|V^\top (p_j - p_i)\|}{\tilde{S}(p_i,p_j)}\), Eq.~\eqref{eq:tilderij} \\
$\psi_{ij}$ & Projected direction norm, \(\left\|V^\top \frac{p_j - p_i}{\|p_j - p_i\|}\right\|\), Eq.~\eqref{eq:psi_phi_star} \\
$\phi^\star_{ij}$ & Minimum edge projection norm on \(\widetilde{P}_{ij}\), Eq.~\eqref{eq:psi_phi_star} \\
$\phi_G$ & Minimum edge projection norm over all edges, \(\min_{e \in E(G)} \left\|V^\top \frac{e}{\|e\|}\right\|\) \\
$\mu_k$ & Average-case spectral capture, \(\frac{\sum_{\ell \le k} \lambda_\ell}{\sum_{\ell} \lambda_\ell}\), Eq.~\eqref{eq:avg_energy} \\
$Z$ & Squared norm of random projected direction, \(\|V^\top U\|^2\), Eq.~\eqref{eq:avg_energy} \\
$\widehat{C}^{\mathrm{sp}\prime}$ & Post-projection non-convexity index (fixed pairs), \(\frac{1}{|\mathcal{D}^*|} \sum_{(i,j) \in \mathcal{D}^*} \tilde{r}_{ij}\) \\
$\mathcal{D}^{*\prime}$ & Reselected admissible pairs, \(\{(i,j): \tilde{r}_{ij} \le \tau\}\) \\
$\widehat{C}^{\mathrm{sp}\prime\prime}$ & Post-projection non-convexity index (reselected pairs), \(\frac{1}{|\mathcal{D}^{*\prime}|} \sum_{(i,j) \in \mathcal{D}^{*\prime}} \tilde{r}_{ij}\) \\
\bottomrule
\end{tabular}
}
\end{table}

\begin{table}[t]
\centering
\caption{Comparison of dimensionality-reduction methods by primary optimization target, nature of theoretical guarantees, and linearity (linear or nonlinear).}
\label{tab:comparison}
\scalebox{0.75}{
\begin{tabular}{|l|l|l|l|}
\toprule
Method & Target & Guarantees & Linearity \\
\midrule
PCA \cite{b8} & Variance & Perturbation (classical) & Linear \\
LPP/NPP/OLPP \cite{heNiyogi2004lpp,he2005npe,cai2006olpp} & Laplacian smoothness & Spectral properties & Linear \\
Isomap \cite{isomap} & Geodesic distances & Isometry (idealized) & Nonlinear \\
t-SNE \cite{tsne} & Local KL divergence & None & Nonlinear \\
UMAP \cite{umap} & Fuzzy set objective & None & Nonlinear \\
CDP (ours) & Detour / non-convexity directions & Pairwise certificate; average-case spectral & Linear \\
\bottomrule
\end{tabular}
}
\end{table}

\section{Preliminaries}
\label{sec:prelim}

Let $\{p_i\}_{i=1}^N\subset\mathbb{R}^d$ and let $G=(V,E)$ be a connected $k$-NN graph with Euclidean edge weights. For a pair $(i,j)$ define the convexity (detour) ratio
\begin{equation}\label{eq:rij}
r_{ij} \;=\; \frac{\|p_j-p_i\|}{S(p_i,p_j)},\qquad 0<r_{ij}\le 1,
\end{equation}
where $S(p_i,p_j)$ is the shortest-path distance in $G$. A pair is admissible if $r_{ij}\le \tau$ with $\tau<1$, and we write $\mathcal D^*=\{(i,j): r_{ij}\le \tau\}$. We summarize detour prevalence by the admissible non-convexity index
\begin{equation}
\widehat C^{\mathrm{sp}} \;=\; \frac{1}{|\mathcal D^*|}\sum_{(i,j)\in\mathcal D^*} r_{ij}.
\end{equation}
Lower $\widehat C^{\mathrm{sp}}$ indicates stronger average detours among admissible pairs. For $u_{ij}:=(p_j-p_i)/\|p_j-p_i\|$, we define the non-convexity structure matrix as
\begin{equation}
S^{\mathrm{nc}}
\;:=\;
\frac{1}{|\mathcal D^*|}\sum_{(i,j)\in\mathcal D^*} (1-r_{ij})\,u_{ij}u_{ij}^\top.
\label{eq:Snc}
\end{equation}
Note that each summand in~\eqref{eq:Snc} is a rank-1 positive semidefinite projector scaled by $(1-r_{ij})\ge 0$, hence $S^{\mathrm{nc}}$ is positive semidefinite. Let its eigenvalues be $\lambda_1\ge\cdots\ge\lambda_d\ge 0$ with orthonormal eigenvectors $\{z_\ell\}$. CDP uses $V=[z_1,\dots,z_k]\in\mathbb{R}^{d\times k}$, the matrix of the top-$k$ eigenvectors of the non-convexity structure matrix $S^{\rm nc}$, and returns $p_i'=V^\top p_i$.

\section{Convexity-Driven Projection (CDP) Algorithm}
\label{sec:algo}

We present the Convexity-Driven Projection (CDP) algorithm, designed to reduce the dimensionality of point clouds while preserving detour-induced non-convexity. The algorithm constructs a $k$-nearest-neighbor ($k$-NN) graph, identifies admissible pairs based on a detour ratio threshold, builds a non-convexity structure matrix, and projects onto its top-$k$ eigenvectors. The steps are formalized in Algorithm~\ref{alg:cdp}, followed by discussions on computational complexity and graph construction considerations.

The CDP algorithm operates on a sparse $k$-NN graph with $m=\Theta(N k_{\text{nn}})$ edges. Computing all-pairs shortest-path distances (APSP) using multi-source Dijkstra’s algorithm incurs a cost of $\tilde O(N^2 k_{\text{nn}})$, which dominates for large $N$. Constructing the non-convexity structure matrix $S^{\mathrm{nc}}$ from $|\mathcal D^*|$ admissible pairs requires $\Theta(|\mathcal D^*| d^2)$ operations due to rank-1 outer product calculations. The randomized SVD to extract the top-$k$ eigenvectors of the $d \times d$ matrix $S^{\mathrm{nc}}$ costs $\tilde O(d^2 \ell)$, where $\ell$ is the number of iterations, typically small; exact SVD would cost $O(d^3)$. Finally, projecting $N$ points onto the $k$-dimensional subspace takes $O(N d k)$. For large datasets, the APSP step may be a bottleneck, but approximations (e.g., landmark-based shortest paths) could reduce this cost in practice.

The algorithm uses a mutual $k$-NN graph, where an edge exists only if both points are among each other’s $k_{\text{nn}}$ nearest neighbors, to mitigate shortcuts in regions of varying point density. Edge weights are Euclidean distances, preserving the ambient geometry. Connectivity is ensured by selecting $k_{\text{nn}} \gtrsim \log N$ under near-homogeneous sampling, as supported by random geometric graph theory \cite{Penrose2003}. In heterogeneous datasets, we retain the giant component to maintain connectivity. The threshold $\tau < 1$ controls the selection of admissible pairs, balancing the capture of strong detours (small $\tau$) with sufficient pair coverage (larger $\tau$). 

The CDP algorithm's performance is influenced by the choice of hyperparameters $\tau$ and $k_{\text{nn}}$. The threshold $\tau$ determines the stringency for selecting admissible pairs with a lower $\tau$ focuses on pairs with strong detours (small $r_{ij}$), resulting in fewer but more informative directions in $S^{\mathrm{nc}}$, which may enhance preservation of non-convexity but risks insufficient pairs for a robust matrix, potentially leading to degenerate projections. Conversely, a higher $\tau$ includes more pairs, capturing milder detours and yielding projections that resemble variance-based methods like PCA, with broader coverage but less emphasis on extreme non-convexity. The number of nearest neighbors $k_{\text{nn}}$ affects graph sparsity. Smaller values create sparser graphs with longer shortest paths, amplifying detours and strengthening the non-convexity signal in $S^{\mathrm{nc}}$. In comparison, larger values densify the graph, shortening paths and reducing detour prevalence, which can smooth the projection but dilute its focus on non-convex structures. These parameters can be tuned on a validation set to balance pair selection and detour preservation.

\begin{algorithm}[t]
\caption{Convexity-Driven Projection (CDP)}
\label{alg:cdp}
\begin{algorithmic}[1]
\REQUIRE Points $\{p_i\}\subset\mathbb{R}^d$, target $k<d$, $k_{\text{nn}}$, threshold $\tau$.
\ENSURE $V\in\mathbb{R}^{d\times k}$, projections $\{p_i'\}\subset\mathbb{R}^k$
\STATE Standardize coordinates (zero mean, unit variance).
\STATE Build a (mutual) $k$-NN graph $G$ with Euclidean edge weights.
\STATE Compute all-pairs shortest-path distances $S(p_i,p_j)$ on $G$ (multi-source Dijkstra).
\STATE Form $\mathcal D^*=\{(i,j): r_{ij}\le \tau\}$.
\STATE Build $S^{\mathrm{nc}}$ via Eq.~\eqref{eq:Snc}.
\STATE Compute top-$k$ eigenvectors of $S^{\mathrm{nc}}$.
\STATE Project: $p_i'=V^\top p_i$.
\STATE Projected graph: reuse the edge set $E$; assign edge weights $w'(u,v)=\|V^\top (p_v-p_u)\|$.
\RETURN $V$, $\{p_i'\}$ and the projected graph $(V,E,w')$.
\end{algorithmic}
\end{algorithm}

\section{CDP: A-Posteriori Certificate and Average-Case Spectral Bound}
\label{sec:theory}
We give (i) a pairwise, data-dependent certificate valid after fitting $V$, and (ii) an average-case spectral bound explaining typical behavior and enabling quantile statements.
\subsection{A-Posteriori Pairwise Certificate}
For an admissible pair \((i,j) \in \mathcal{D}^*\), let \(\tilde{r}_{ij}\) denote the post-projection convexity ratio, defined as
\begin{equation}
\tilde{r}_{ij} =` \frac{\|V^\top (p_j - p_i)\|}{\tilde{S}(p_i, p_j)},
\label{eq:tilderij}
\end{equation}
where \(\|V^\top (p_j - p_i)\|\) is the Euclidean distance between the projected points \(p_i' = V^\top p_i\) and \(p_j' = V^\top p_j\) in \(\mathbb{R}^k\), and \(\tilde{S}(p_i, p_j)\) is the shortest-path distance on the projected graph \((V, E, w')\) with edge weights \(w'(u, v) = \|V^\top (p_v - p_u)\|\) (see Algorithm~\ref{alg:cdp}, Step 8). This mirrors the pre-projection ratio \(r_{ij} = \frac{\|p_j - p_i\|}{S(p_i, p_j)}\) from Eq.~\eqref{eq:rij}, but computed in the reduced space. Fix $(i,j)\in\mathcal D^*$. Define
\begin{equation}
\psi_{ij} \;:=\; \left\|V^\top \frac{p_j-p_i}{\|p_j-p_i\|}\right\|\in[0,1],
\qquad
\phi^{\star}_{ij} \;:=\; \min_{e\in \widetilde P_{ij}}
\left\|V^\top \frac{e}{\|e\|}\right\|\in[0,1],
\label{eq:psi_phi_star}
\end{equation}
where $\widetilde P_{ij}$ is a projected shortest path between $p_i'$ and $p_j'$ in $(V,E,w')$, and the minimum is over its edges $e$ expressed as vectors in the original coordinates.

To establish the pairwise a-posteriori certificate, we bound the distortion ratio \(\tilde{r}_{ij}/r_{ij}\) for each admissible pair \((i,j)\) using measurable cosines computed after the projection matrix \(V\) is fitted. Intuitively, the lower bound arises because the orthogonal projection preserves or reduces distances, ensuring the projected shortest path is no longer than the original, while the upper bound leverages the specific geometry of the projected shortest path to prevent excessive collapse of detours. More precisely, decompose the ratio as the product of the projected Euclidean factor \(\psi_{ij}\) (the cosine of the pair's direction onto the subspace) and the original-to-projected shortest-path ratio \(S/\tilde{S}\). Since the projection is non-expansive, every edge shortens or stays the same, implying \(\tilde{S} \leq S\) and thus \(S/\tilde{S} \geq 1\), yielding the lower bound \(\psi_{ij} \leq \tilde{r}_{ij}/r_{ij}\). For the upper bound, consider a projected shortest path \(\tilde{P}_{ij}\): its length \(\tilde{S}\) is at least \(\phi^*_{ij}\) times its original length \(L_{\mathrm{orig}}(\tilde{P}_{ij})\), where \(\phi^*_{ij}\) is the minimum cosine over its edges; since \(L_{\mathrm{orig}}(\tilde{P}_{ij}) \geq S\) (as \(S\) is the minimal original path length), it follows that \(\tilde{S} \geq \phi^*_{ij} S\), so \(S/\tilde{S} \leq 1/\phi^*_{ij}\), completing the bound. This approach ensures the certificate is verifiable per pair, relying only on post-projection computations like recovering \(\tilde{P}_{ij}\) via Dijkstra's algorithm with parent pointers. The following theorem makes this precise. 
\begin{theorem}[Pairwise a-posteriori certificate]
\label{thm:pairwise-corrected}
For every admissible pair $(i,j)\in\mathcal D^*$,
\begin{equation}
\psi_{ij} \;\le\; \frac{\tilde r_{ij}}{r_{ij}} \;\le\; \frac{1}{\phi^{\star}_{ij}}.
\label{eq:pair_cert_bounds_corrected}
\end{equation}
\end{theorem}
\begin{proof}
Observe that $\displaystyle \frac{\tilde r_{ij}}{r_{ij}}=\frac{\|V^\top (p_j-p_i)\|}{\|p_j-p_i\|}\cdot \frac{S(p_i,p_j)}{\tilde S(p_i,p_j)}$.
Since, orthogonal projection is non-expansive, $\tilde S\le S$ and $\|V^\top (p_j-p_i)\|/\|p_j-p_i\|=\psi_{ij}$, giving the lower bound. For the upper bound, let $\widetilde P_{ij}$ be a projected shortest path. Then \[ \tilde S=\sum_{e\in \widetilde P_{ij}}\|V^\top e\|\ge \phi^{\star}_{ij}\sum_{e\in \widetilde P_{ij}}\|e\|=\phi^{\star}_{ij}L_{\mathrm{orig}}(\widetilde P_{ij})\ge \phi^{\star}_{ij}S,\] hence $S/\tilde S\le 1/\phi^{\star}_{ij}$.
\end{proof}
As an immediate consequence of Theorem~\ref{thm:pairwise-corrected}, we derive a uniform upper bound that applies to all admissible pairs simultaneously, by considering the minimum captured cosine across all edges in the graph. While simpler to compute globally, this bound may be conservative compared to the per-pair version.
\begin{corollary}[Uniform graph-wise bound]
\label{cor:graphwise}
Let $\phi_G := \min_{e\in E(G)} \big\|V^\top \frac{e}{\|e\|}\big\|$. Then for all admissible pairs,
$\displaystyle \frac{\tilde r_{ij}}{r_{ij}} \le \frac{1}{\phi_G}$.
\end{corollary}
\noindent
This uniform bound can be loose when $\phi_G$ is small; Theorem~\ref{thm:pairwise-corrected} is tighter per pair.

\paragraph{Computing $\phi^{\star}_{ij}$.}
During Dijkstra on the projected graph, store parents to recover $\widetilde P_{ij}$; set $\phi^{\star}_{ij}=\min_{e\in \widetilde P_{ij}}\|V^\top e\|/\|e\|$. After the run, this adds $O(|\widetilde P_{ij}|\,d)$ per pair.

\subsection{Average-Case Spectral Bound (Prior, Aggregate Behavior)}
Let $U$ be the random unit direction drawn from the weighted empirical distribution over admissible pairs. Then
\[
\mathbb{P}\{U=u_{ij}\}\propto(1-r_{ij}),\qquad u_{ij}=\frac{p_j-p_i}{\|p_j-p_i\|}.
\]
Then $S^{\mathrm{nc}}=\mathbb{E}[UU^\top]$ up to normalization, and
\begin{equation}
\mu_k \;:=\; \mathbb{E}\big[\|V^\top U\|^2\big] \;=\; \frac{\sum_{\ell\le k}\lambda_\ell}{\sum_{\ell}\lambda_\ell}.
\label{eq:avg_energy}
\end{equation}
Let $Z=\|V^\top U\|^2\in[0,1]$ with $\mathbb{E}[Z]=\mu_k$. For any $a\in(0,1)$,
\begin{equation}
\mathbb{P}\{Z \ge 1-a\} \;\ge\; 1 - \frac{1-\mu_k}{a}.
\label{eq:markov_lower_quantile}
\end{equation}
For a random admissible pair, the Euclidean factor equals $\sqrt{Z}$; combined with Theorem~\ref{thm:pairwise-corrected}, we have
\[
\sqrt{Z}\ \le\ \tilde r_{ij}/r_{ij}\ \le\ 1/\phi^{\star}_{ij},
\]
so empirical distributions of $\sqrt{Z}$ and $\phi^{\star}_{ij}$ control typical multiplicative distortion.

\subsection{Verification Protocol and Metrics}
To evaluate CDP’s effectiveness in preserving detour-induced non-convexity, we define metrics and quantiles that assess projection quality and verify theoretical guarantees. These include detour errors and certificate bounds, which are detailed below.
\begin{enumerate}[(i)]
\item Fixed-pairs detour error: $|\widehat C^{\mathrm{sp}}-\widehat C^{\mathrm{sp}\prime}|/\widehat C^{\mathrm{sp}}$, where $\widehat C^{\mathrm{sp}\prime}$ is computed post-projection on the same admissible pairs $\mathcal D^\star$ using the projected graph. (This decouples pair selection from projection.)

\item Reselected-pairs detour error: compute $\mathcal D^{\star\prime}=\{(i,j): \tilde r_{ij}\le \tau\}$ on the projected graph and report $|\widehat C^{\mathrm{sp}}-\widehat C^{\mathrm{sp}\prime\prime}|/\widehat C^{\mathrm{sp}}$, where $\widehat C^{\mathrm{sp}\prime\prime}$ averages $\tilde r_{ij}$ over $\mathcal D^{\star\prime}$. This reflects the new detour geometry.

\item Certificate quantiles (a-posteriori): for $(\psi_{ij},\phi^{\star}_{ij})$ in \eqref{eq:psi_phi_star}, report
$q_{0.10}(\psi)$ and $q_{0.90}(1/\phi^{\star})$ across $\mathcal D^\star$.
Thus, for at least $90\%$ of admissible pairs,
\[
q_{0.10}(\psi)\ \le\ \frac{\tilde r_{ij}}{r_{ij}}\ \le\ q_{0.90}(1/\phi^{\star}).
\]`

\item Average-case spectral capture: $\mu_k$ from~\eqref{eq:avg_energy}.
\end{enumerate}

\section{Experiments}
\label{sec:results}

We evaluate the Convexity-Driven Projection (CDP) algorithm’s effectiveness in preserving detour-induced non-convexity using a toy example and benchmark datasets, including Swiss Roll, Torus, S-Curve, Helix, M\"{o}bius Strip, Klein Bottle, and Annulus with Obstacle. Comparing CDP against PCA, UMAP, and LPP, we report visualizations demonstrating CDP's performance on non-convex structures.

\subsection{Toy Example: Five Points ($d=3 \rightarrow k=2$)}
\label{subsec:toy}

Consider five points
\[
\begin{array}{c|ccc}
\text{Label} & x & y & z \\ \hline
A & 0 & 0 & 0 \\
B & 1 & 0.2 & 0 \\
C & 2 & 0 & 0 \\
D & 2 & 1 & 0 \\
E & 1 & 0.5 & 1
\end{array}
\]
Using Euclidean edge weights, we build a mutual $k$-NN graph with $k=2$. The mutual edges are
$(A,B)$, $(A,E)$, $(B,C)$, $(C,D)$, so the graph is connected via the chain $D\!-\!C\!-\!B\!-\!A\!-\!E$. Euclidean distances $\|p_j-p_i\|$ and graph shortest paths $S(p_i,p_j)$ yield the convexity ratios
\[
r_{ij}=\frac{\|p_j-p_i\|}{S(p_i,p_j)}\in(0,1].
\]
For threshold $\tau=0.75$, the pairs with $r_{ij}\le \tau$ are the admissible set $\mathcal D^*$. Table~\ref{tab:toy_pairs} lists all unordered pairs.
\begin{table}[h]
\centering
\caption{Pairwise Euclidean, shortest-path, and $r_{ij}$; admissible if $r_{ij}\le \tau=0.75$.}
\label{tab:toy_pairs}
\scalebox{0.95}{
\begin{tabular}{cc|c|c|c|c}
\toprule
$i$ & $j$ & $\|p_j-p_i\|$ & $S(p_i,p_j)$ & $r_{ij}$ & admissible \\
\midrule
A & B & 1.0198 & 1.0198 & 1.0000 & FALSE \\
A & C & 2.0000 & 2.0396 & 0.9806 & FALSE \\
A & D & 2.2361 & 3.0396 & 0.7356 & TRUE \\
A & E & 1.5000 & 1.5000 & 1.0000 & FALSE \\
B & C & 1.0198 & 1.0198 & 1.0000 & FALSE \\
B & D & 1.2806 & 2.0198 & 0.6340 & TRUE \\
B & E & 1.0440 & 2.5198 & 0.4143 & TRUE \\
C & D & 1.0000 & 1.0000 & 1.0000 & FALSE \\
C & E & 1.5000 & 3.5396 & 0.4238 & TRUE \\
D & E & 1.5000 & 4.5396 & 0.3304 & TRUE \\
\bottomrule
\end{tabular}
}
\end{table}
Thus $|\mathcal D^*|=5$ and the admissible non-convexity index is
\[
\widehat C^{\mathrm{sp}}=\frac{1}{5}\sum_{(i,j)\in\mathcal D^*} r_{ij}=0.5076.
\]
For the non-convexity structure matrix and spectrum, we use
\[
S^{\mathrm{nc}}=\frac{1}{|\mathcal D^*|}\sum_{(i,j)\in\mathcal D^*} (1-r_{ij})\,u_{ij}u_{ij}^\top,\quad
u_{ij}=\frac{p_j-p_i}{\|p_j-p_i\|},
\]
to obtain
\[
S^{\mathrm{nc}}=
\begin{pmatrix}
0.197665 & 0.061001 & -0.110738\\
0.061001 & 0.076493 & 0.028090\\
-0.110738 & 0.028090 & 0.218200
\end{pmatrix}.
\]
Its eigenvalues (ascending) are $0.021002,\ 0.150383,\ 0.320973$, with corresponding eigenvectors (columns, descending order of eigenvalues)
\[
V=\begin{pmatrix}
-0.689630 & 0.512628 \\
-0.089508 & 0.640566 \\
\ \ 0.718609 & 0.571741
\end{pmatrix}\in\mathbb{R}^{3\times 2}.
\]
The average-case spectral capture is
\[
\mu_k=\frac{\lambda_1+\lambda_2}{\lambda_1+\lambda_2+\lambda_3}=0.9573,
\]
and the projected coordinates $p_i'=V^\top p_i$ are given by 
\[
\begin{array}{c|cc}
\text{Label} & u_1 & u_2 \\ \hline
A & 0.000000 & 0.000000 \\
B & -0.707531 & 0.640741 \\
C & -1.379259 & 1.025255 \\
D & -1.468767 & 1.665821 \\
E & \ \ 0.015774 & 1.404652
\end{array}
\]
Figure~\ref{fig:toy_example_plots} visualizes the original 3D points and their 2D projection, highlighting the mutual $k$-NN edges ($A-B$, $A-E$, $B-C$, $C-D$).
\begin{figure}[h]
\centering
\begin{subfigure}[t]{0.48\textwidth}
\centering
\includegraphics[width=\textwidth]{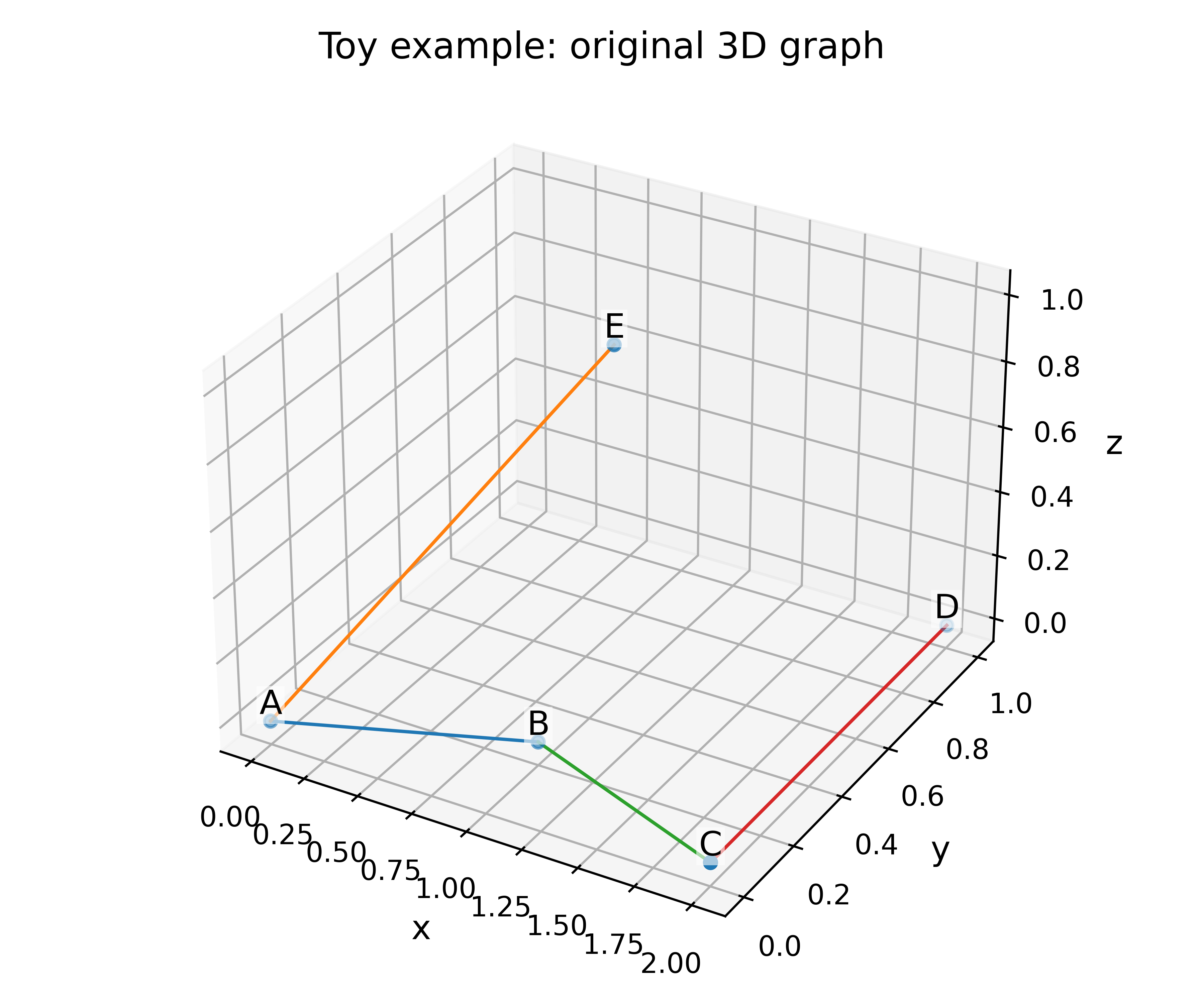}
\caption{3D original points with mutual $k$-NN edges ($A-B$, $A-E$, $B-C$, $C-D$).}
\label{fig:3d_toy}
\end{subfigure}
\hfill
\begin{subfigure}[t]{0.42\textwidth}
\centering
\includegraphics[width=\textwidth]{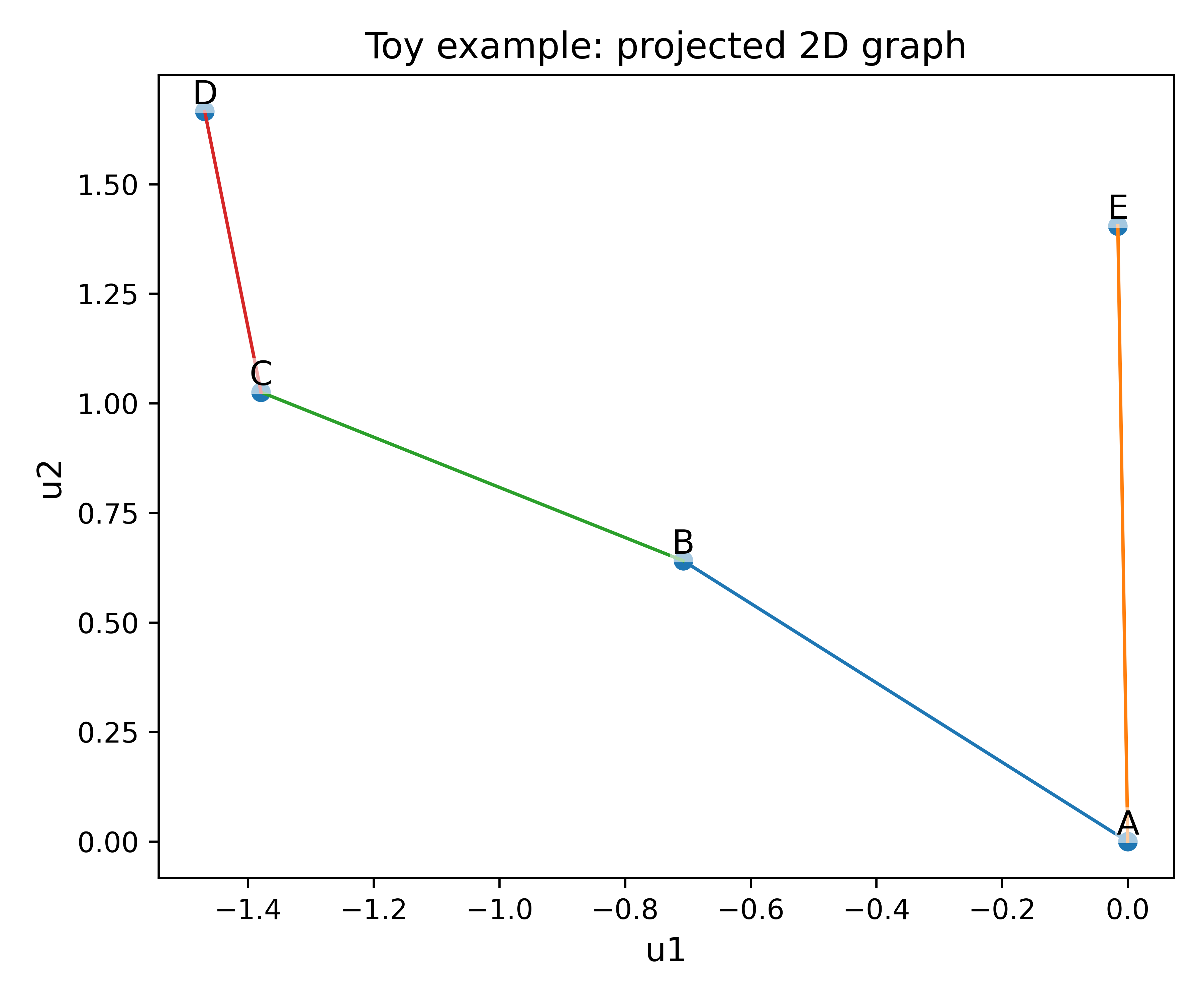}
\caption{2D projected points with mutual $k$-NN edges ($A-B$, $A-E$, $B-C$, $C-D$).}
\label{fig:2d_toy}
\end{subfigure}
\caption{Visualization of the toy example: original 3D points (left) and their 2D projection (right), with mutual $k$-NN edges ($k=2$). Points are labeled ($A$, $B$, $C$, $D$, $E$), and edges reflect the connected graph structure.}
\label{fig:toy_example_plots}
\end{figure}
We reuse the edge set and assign projected edge weights $w'(u,v)=\|V^\top(p_v-p_u)\|$ to obtain
$w'(A\!-\!B)=0.95465$, $w'(A\!-\!E)=1.40474$, $w'(B\!-\!C)=0.77399$, and $w'(C\!-\!D)=0.64679$. For each $(i,j)\in\mathcal D^*$ we compute
\[
\psi_{ij}=\left\|V^\top\frac{p_j-p_i}{\|p_j-p_i\|}\right\|,\qquad
\phi^{\star}_{ij}=\min_{e\in \widetilde P_{ij}}\left\|V^\top\frac{e}{\|e\|}\right\|,
\]
where $\widetilde P_{ij}$ is a projected shortest path (with weights $w'$). Table~\ref{tab:toy_cert} reports the values. In every case certificate~\eqref{eq:pair_cert_bounds_corrected} holds.
\begin{table}[h]
\centering
\caption{Per-pair certificate for the toy example. Here \(\tilde{r}_{ij}\) (Eq.~\eqref{eq:tilderij}) uses the projected graph shortest-path in the denominator; ``Path'' is the projected shortest path (vertex labels).}
\label{tab:toy_cert}
\scalebox{0.95}{
\begin{tabular}{cc|c|c|c|c|c|c|l}
\toprule
$i$ & $j$ & $\psi_{ij}$ & $\phi^\star_{ij}$ & $1/\phi^\star_{ij}$ & $r_{ij}$ & $\tilde{r}_{ij}$ & $\tilde{r}_{ij}/r_{ij}$ & Path \\
\midrule
A & D & 0.993201 & 0.646789 & 1.546099 & 0.735644 & 0.934972 & 1.270958 & A$\to$B$\to$C$\to$D \\
B & D & 0.997029 & 0.646789 & 1.546099 & 0.634034 & 0.898672 & 1.417387 & B$\to$C$\to$D \\
B & E & 0.987113 & 0.936005 & 1.068371 & 0.414330 & 0.436818 & 1.054275 & B$\to$A$\to$E \\
C & E & 0.943524 & 0.758966 & 1.317583 & 0.423776 & 0.451695 & 1.065882 & C$\to$B$\to$A$\to$E \\
D & E & 0.984185 & 0.646789 & 1.546099 & 0.330425 & 0.390543 & 1.181941 & D$\to$C$\to$B$\to$A$\to$E \\
\bottomrule
\end{tabular}
}
\end{table}
The uniform graph-wise constant is $\phi_G=\min_{e\in E}\|V^\top e\|/\|e\|=0.64679$ (attained on $C\!-\!D$), hence the conservative bound $\tilde r_{ij}/r_{ij}\le 1/\phi_G=1.54610$ also holds for all pairs. For the fixed-pairs metric, we compare
\[
\widehat C^{\mathrm{sp}}=0.5076
\quad\text{vs.}\quad
\widehat C^{\mathrm{sp}\prime}=\frac{1}{5}\sum_{(i,j)\in \mathcal D^*}\tilde r_{ij}=0.6225,
\]
getting a relative error of $22.63\%$. For the reselected-pairs metric, we rebuild the admissible set after projection, $\mathcal D^{*\prime}=\{(i,j):\tilde r_{ij}\le \tau\}$; here $|\mathcal D^{*\prime}|=3$ and
\[
\widehat C^{\mathrm{sp}\prime\prime}=\frac{1}{|\mathcal D^{*\prime}|}\sum_{(i,j)\in \mathcal D^{*\prime}}\tilde r_{ij}=0.4264,
\qquad
\text{error}=\frac{|\widehat C^{\mathrm{sp}}-\widehat C^{\mathrm{sp}\prime\prime}|}{\widehat C^{\mathrm{sp}}}=16.01\%.
\]
For average-case spectral check we note that with $\mu_k=0.9573$, Markov’s inequality~\eqref{eq:markov_lower_quantile} implies that for any $a>0$, $\mathbb{P}\{Z \ge 1-a\} \ge 1 - \frac{1-\mu_k}{a}$. At the $90\%$ level, choose $a=(1-\mu_k)/0.1=0.427$, so $Z \ge 0.573$ and $\sqrt{Z} \ge 0.7570$ with probability at least $90\%$. Empirically, the 10th percentile of $\psi_{ij}$ over $\mathcal{D}^*$ is $q_{0.10}(\psi)=0.9435$, comfortably above the bound. Likewise, the empirical 90th percentile of $1/\phi^\star_{ij}$ is $q_{0.90}(1/\phi^\star)=1.5461$.

\subsection{Experimental Evaluation on Benchmark Datasets}
We evaluate CDP on benchmark datasets with non-convex manifolds and structures that induce detours. The datasets include Swiss Roll, Torus, S-Curve, Helix, U-Shape, M\"{o}bius Strip, Klein Bottle, and Annulus with Obstacle, generated with 1000 points each. We apply PCA, UMAP, CDP (with \(k=2\), \(k_{\text{nn}}=10\), \(\tau=0.8\)), and LPP. Projections are visualized in 2D and colored by a parameter (e.g., angle or height) to assess structure preservation. Figure~\ref{fig:benchmark_plots} shows the original 3D and 2D projections. 

\begin{figure*}[htbp]
\centering
\begin{tabular}{c|c|c|c|c}
\toprule
Original 3D & PCA 2D & UMAP 2D & CDP 2D & LPP 2D \\
\midrule
\includegraphics[width=0.20\textwidth]{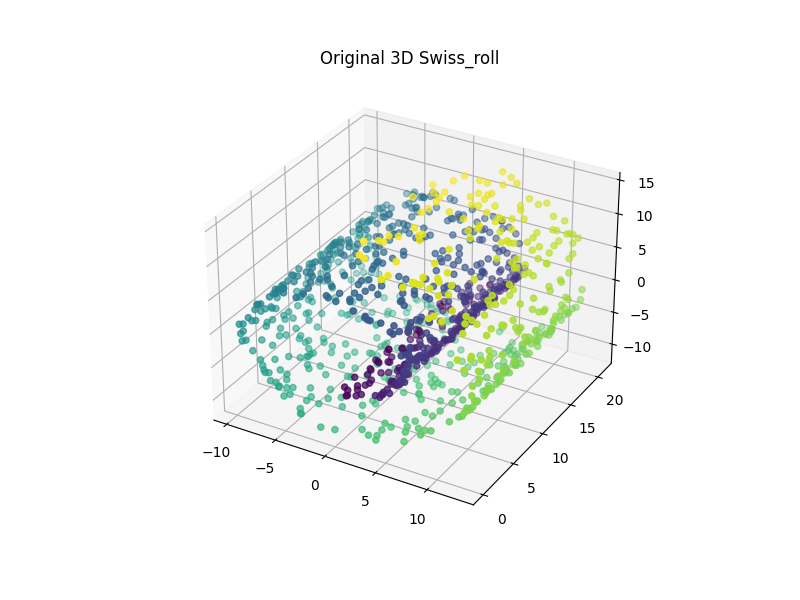} & 
\includegraphics[width=0.20\textwidth]{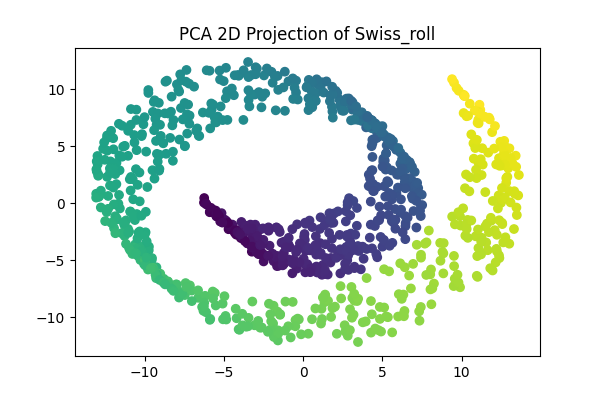} & 
\includegraphics[width=0.20\textwidth]{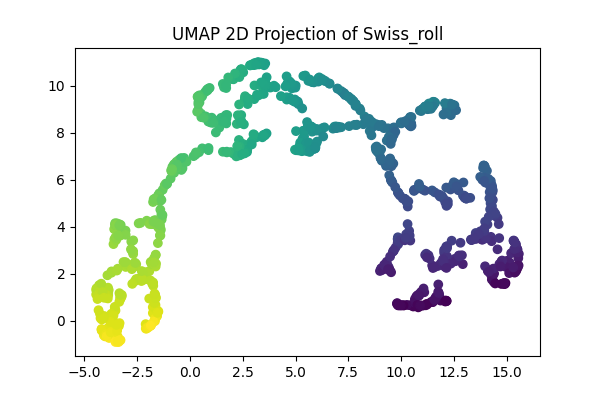} & 
\includegraphics[width=0.20\textwidth]{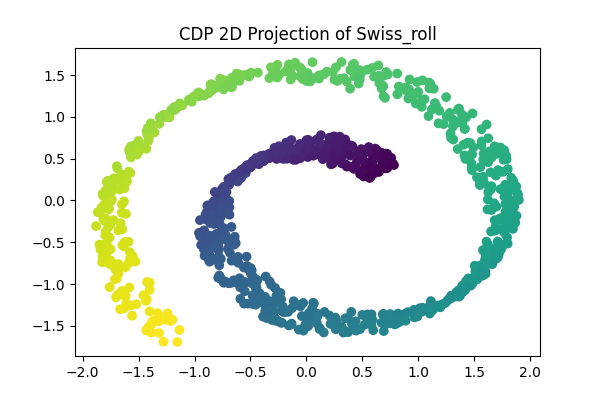} & 
\includegraphics[width=0.20\textwidth]{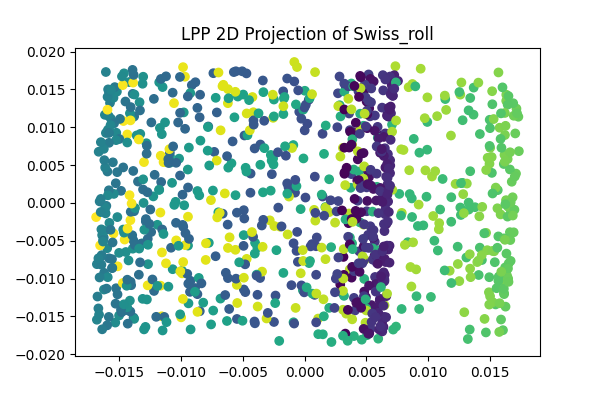} \\
\midrule
\includegraphics[width=0.20\textwidth]{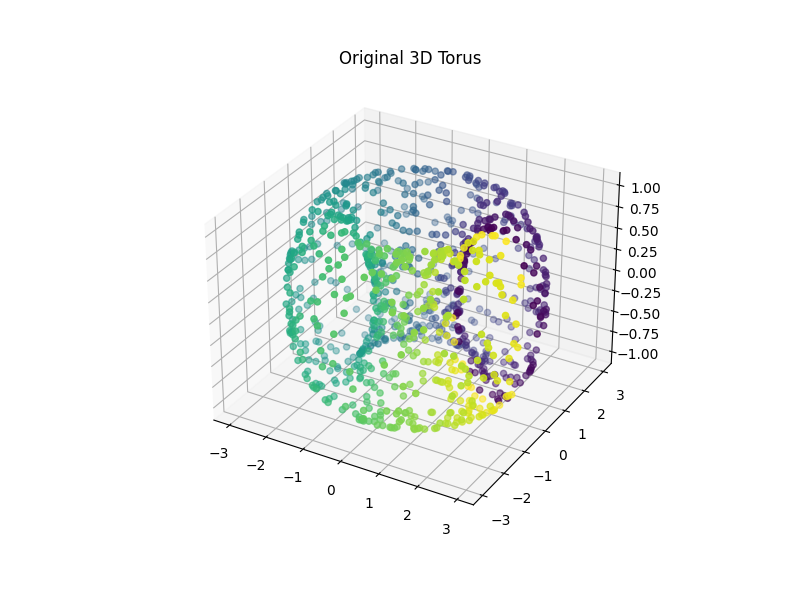} & 
\includegraphics[width=0.20\textwidth]{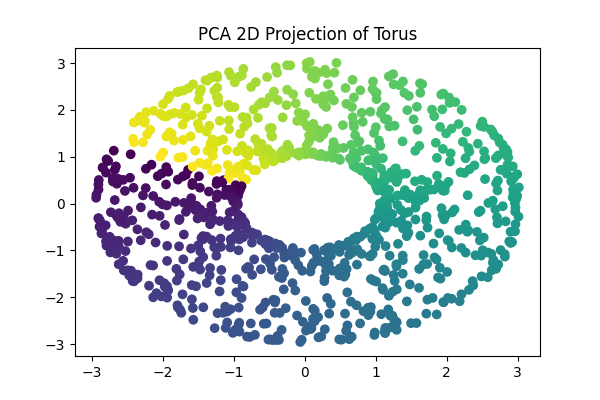} & 
\includegraphics[width=0.20\textwidth]{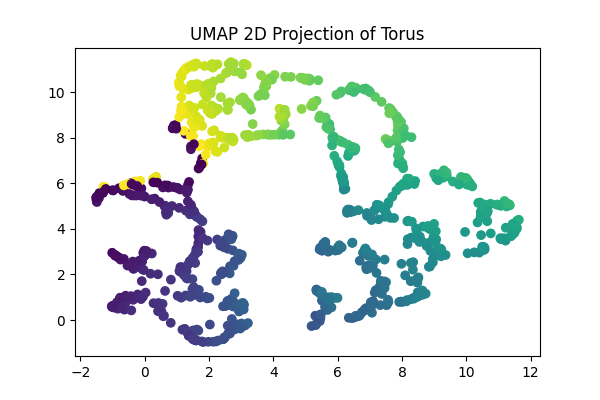} & 
\includegraphics[width=0.20\textwidth]{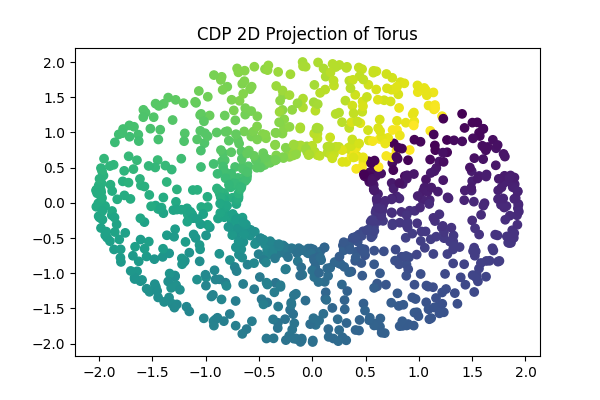} & 
\includegraphics[width=0.20\textwidth]{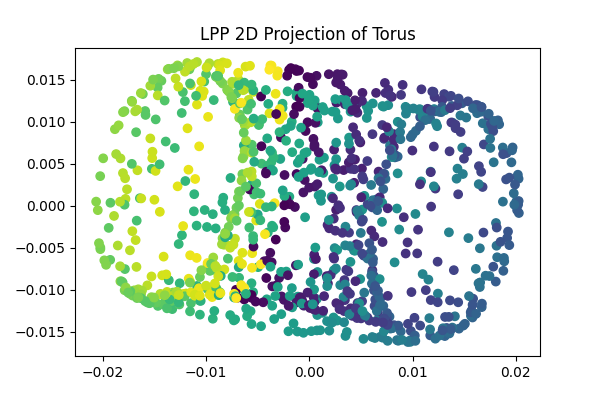} \\
\midrule
\includegraphics[width=0.20\textwidth]{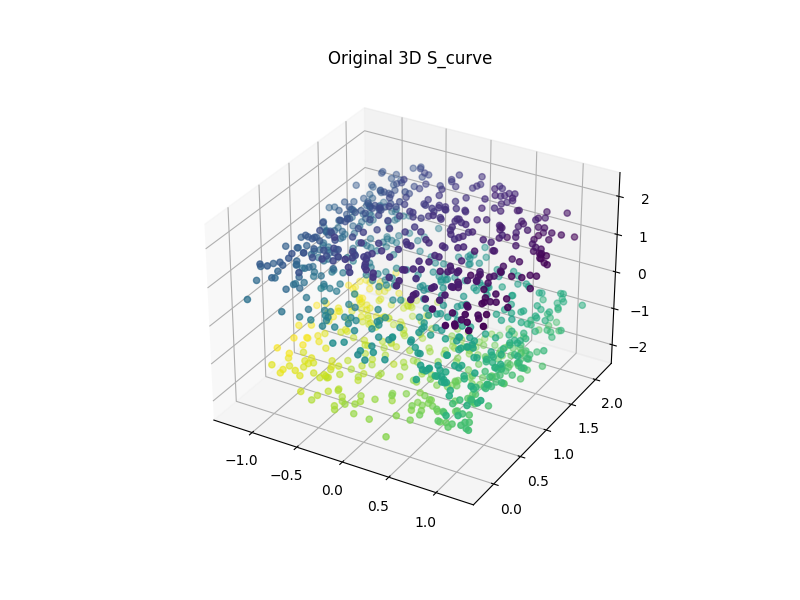} & 
\includegraphics[width=0.20\textwidth]{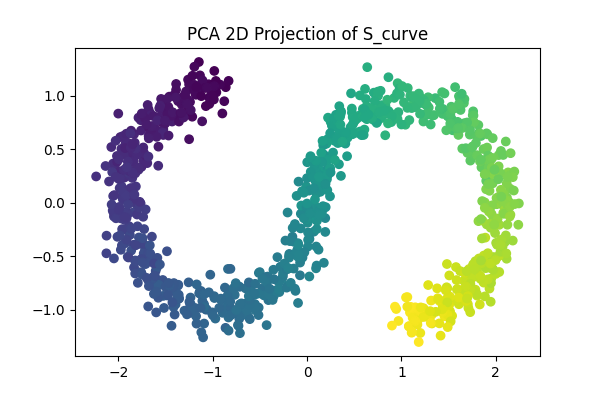} & 
\includegraphics[width=0.20\textwidth]{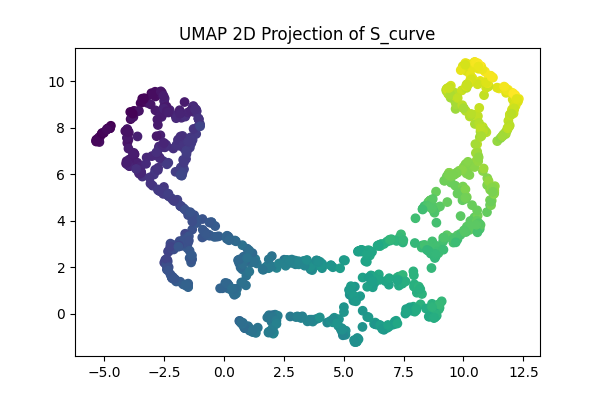} & 
\includegraphics[width=0.20\textwidth]{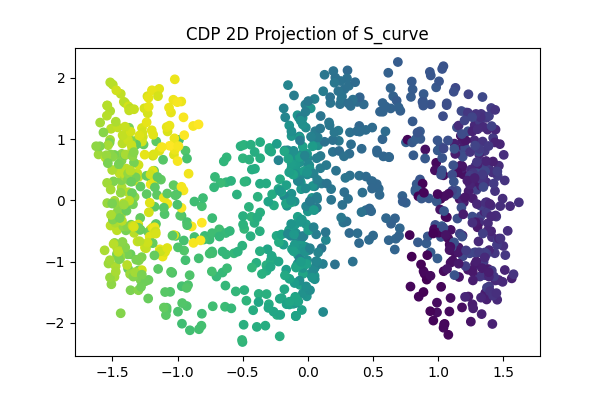} & 
\includegraphics[width=0.20\textwidth]{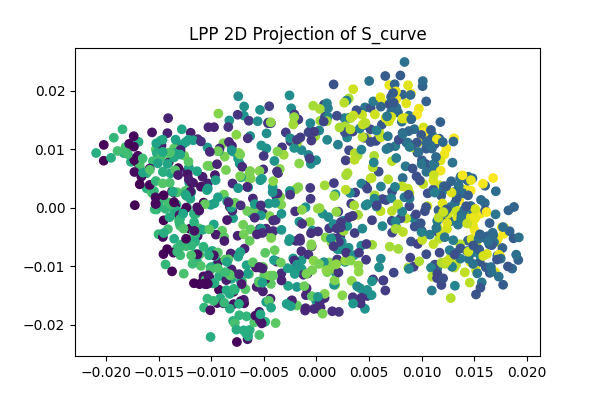} \\
\midrule
\includegraphics[width=0.20\textwidth]{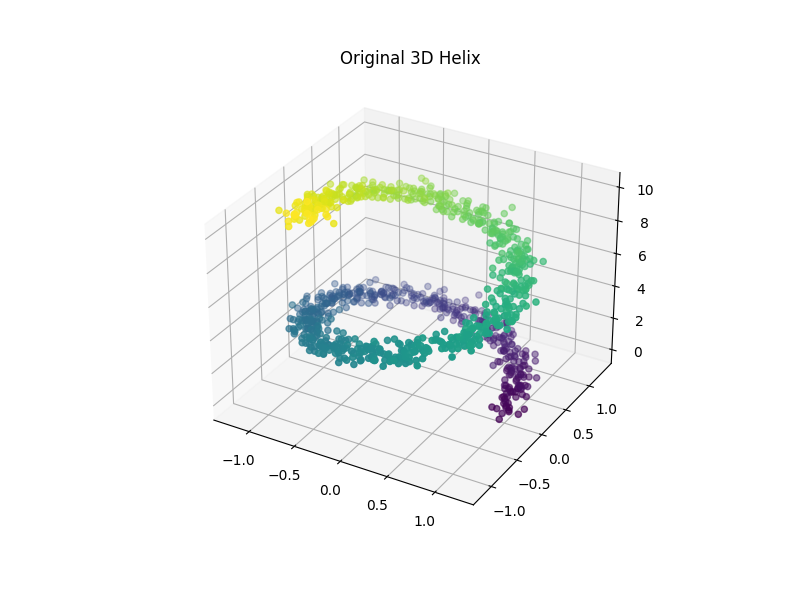} & 
\includegraphics[width=0.20\textwidth]{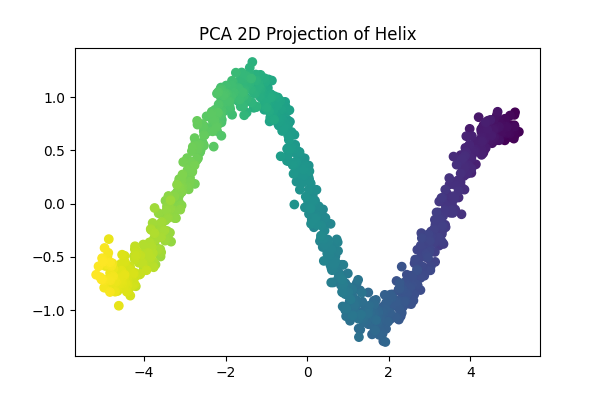} & 
\includegraphics[width=0.20\textwidth]{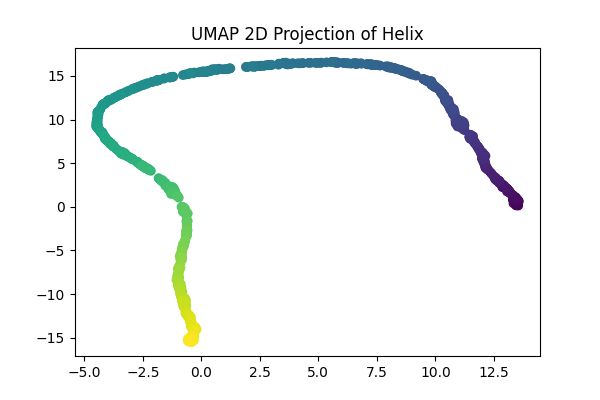} & 
\includegraphics[width=0.20\textwidth]{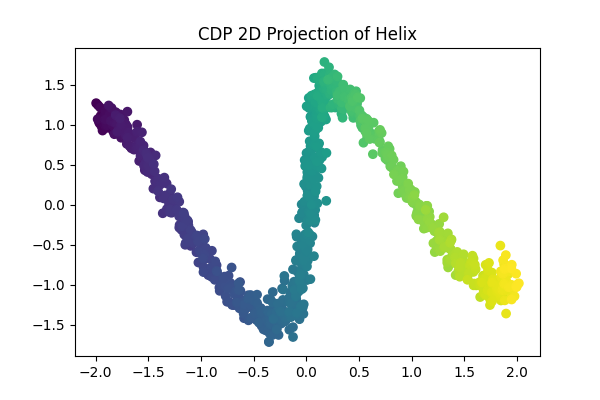} & 
\includegraphics[width=0.20\textwidth]{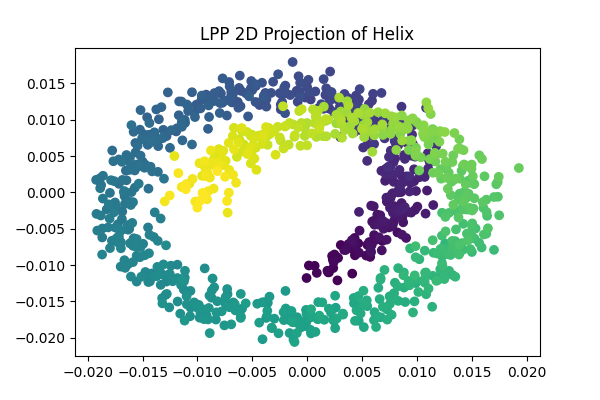} \\
\midrule
\includegraphics[width=0.20\textwidth]{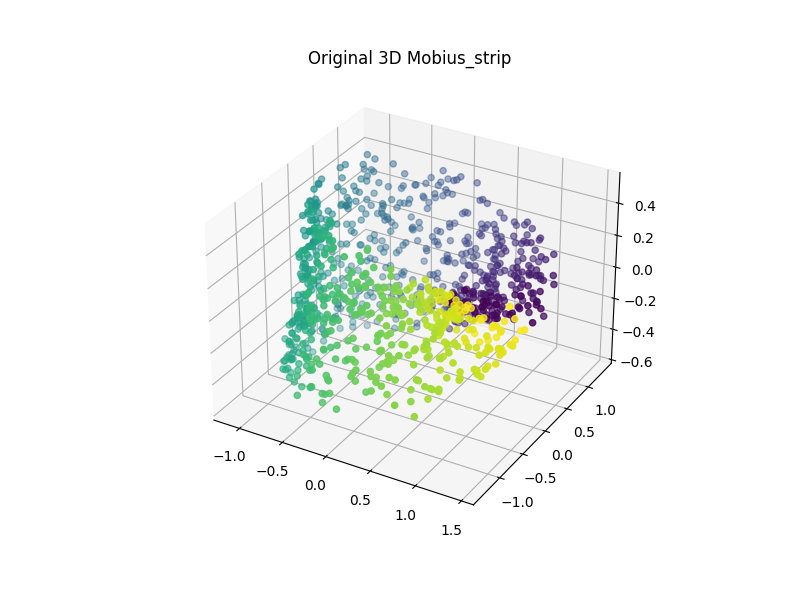} & 
\includegraphics[width=0.20\textwidth]{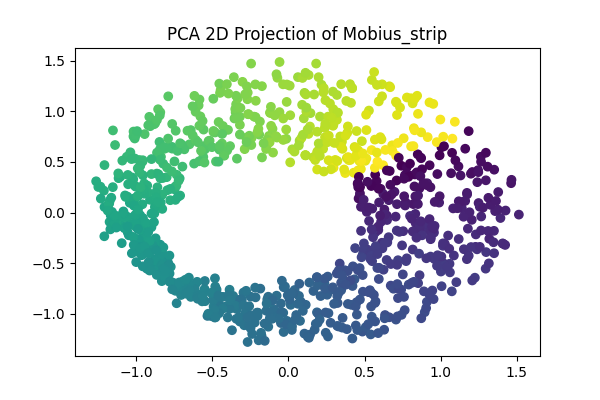} & 
\includegraphics[width=0.20\textwidth]{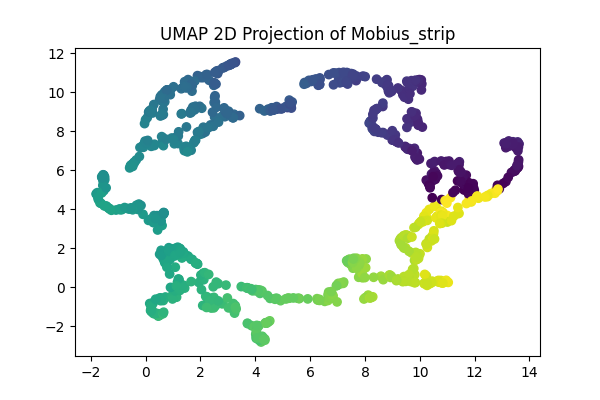} & 
\includegraphics[width=0.20\textwidth]{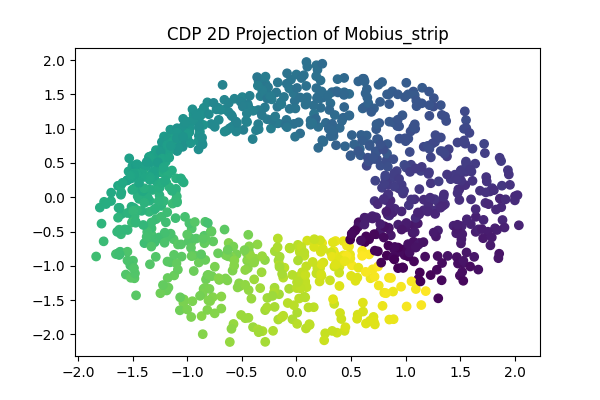} & 
\includegraphics[width=0.20\textwidth]{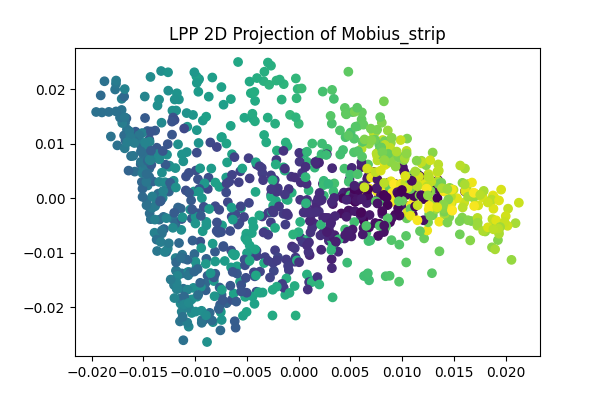} \\
\midrule
\includegraphics[width=0.20\textwidth]{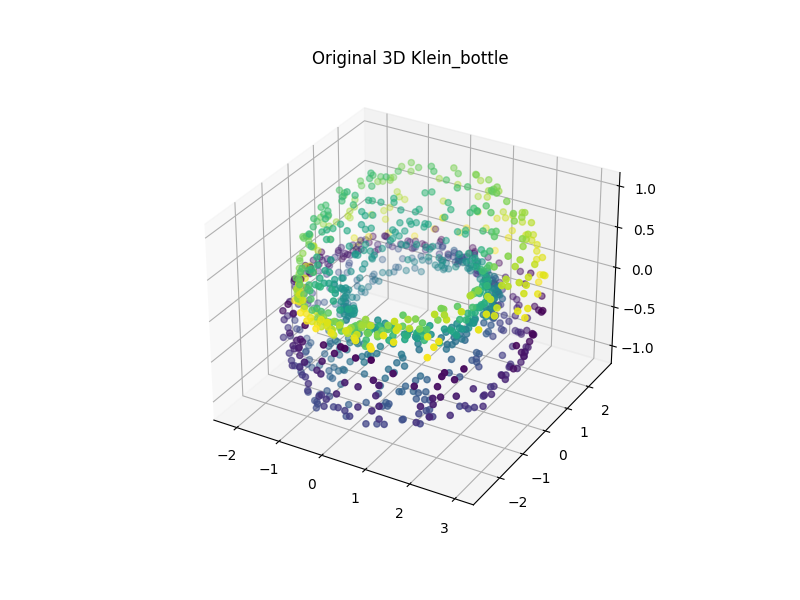} & 
\includegraphics[width=0.20\textwidth]{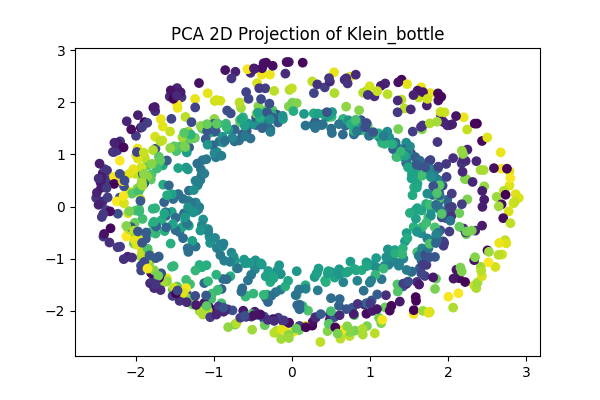} & 
\includegraphics[width=0.20\textwidth]{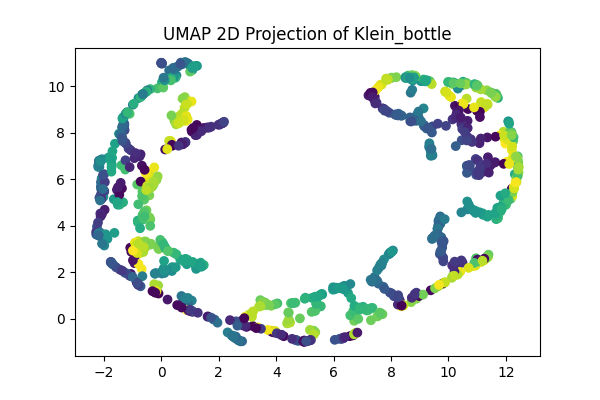} & 
\includegraphics[width=0.20\textwidth]{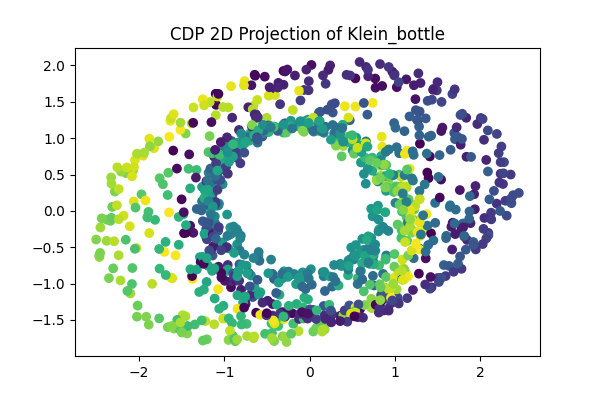} & 
\includegraphics[width=0.20\textwidth]{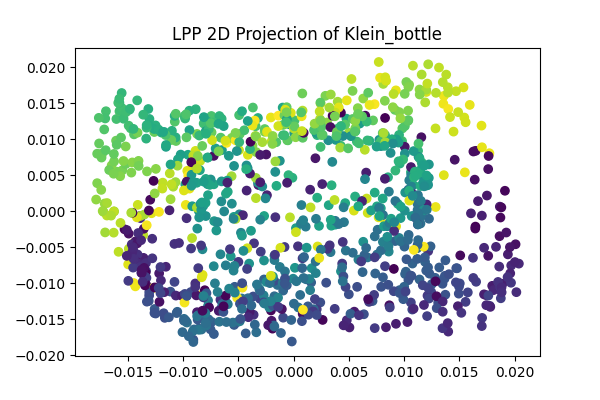} \\
\midrule
\includegraphics[width=0.20\textwidth]{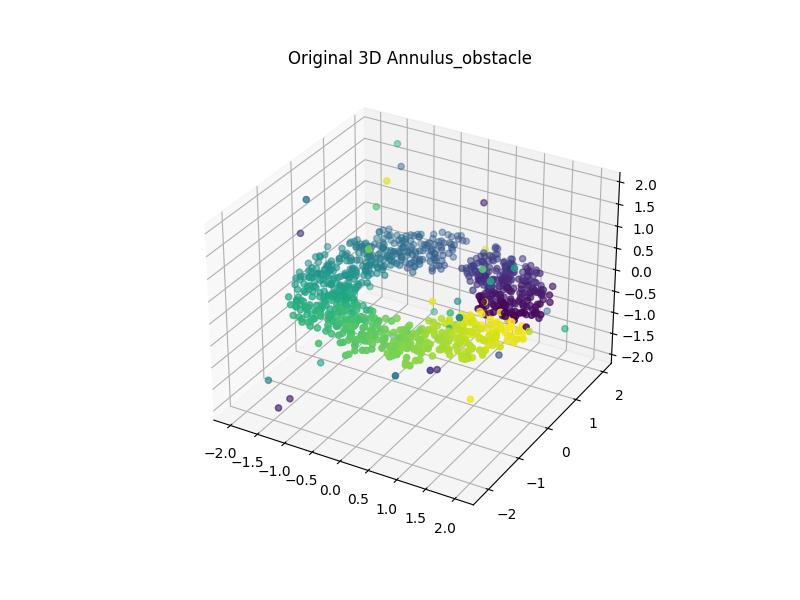} & 
\includegraphics[width=0.20\textwidth]{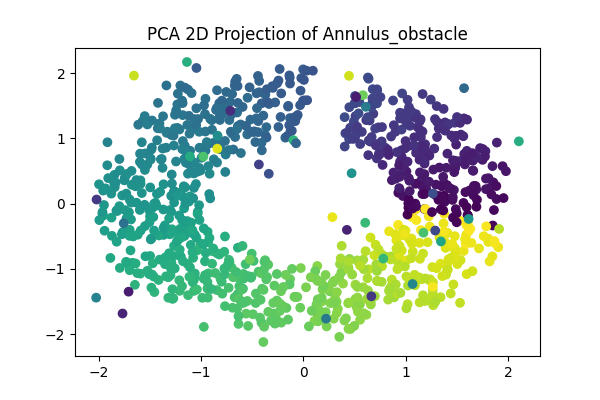} & 
\includegraphics[width=0.20\textwidth]{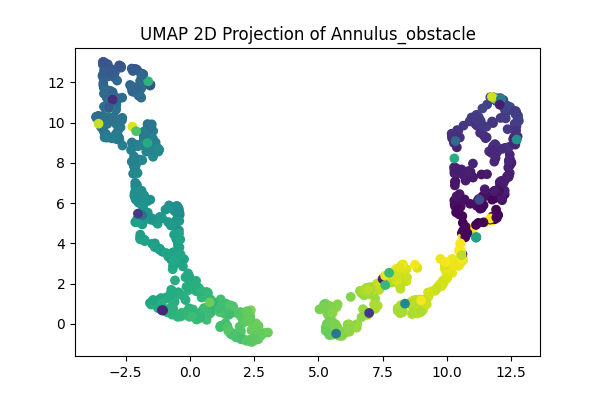} & 
\includegraphics[width=0.20\textwidth]{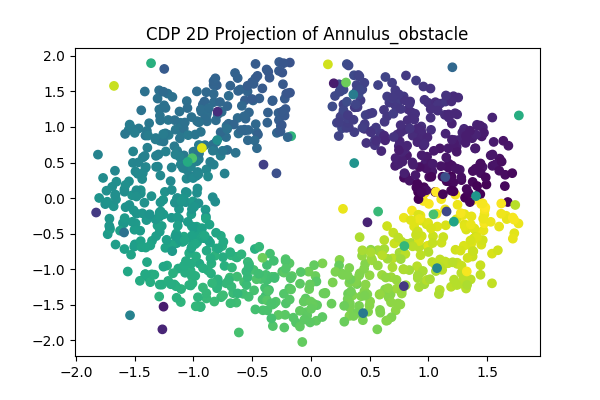} & 
\includegraphics[width=0.20\textwidth]{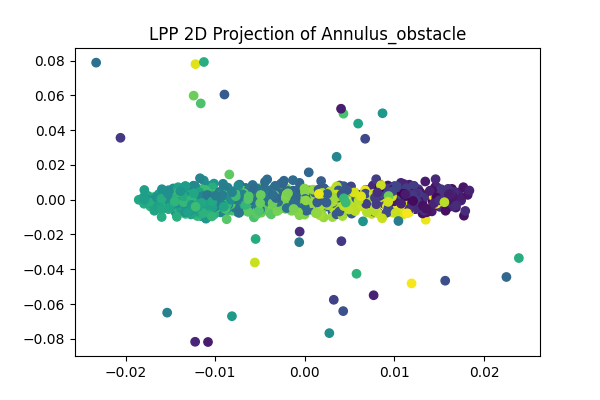} \\
\bottomrule
\end{tabular}
\caption{Datasets from top to bottom: Swiss Roll, Torus, S-Curve, Helix, M\"obius Strip, Klein Bottle, Annulus with Obstacle and methods: Original 3D, PCA 2D, UMAP 2D, CDP 2D, LPP 2D.}
\label{fig:benchmark_plots}
\end{figure*}

\section{Conclusion}
\label{sec:conclusion}

We introduced Convexity-Driven Projection (CDP), a linear dimensionality reduction method that preserves detour-induced local non-convexity in point clouds without requiring boundary estimation. By constructing a positive semidefinite non-convexity structure matrix from admissible pair directions (Section~\ref{sec:prelim}). We provided a pairwise a-posteriori certificate to bound post-projection distortion (Section~\ref{sec:theory}), refined to use projected shortest paths for accuracy, and an average-case spectral bound to quantify typical directional capture. The evaluation protocol, including fixed- and reselected-pairs detour errors and certificate quantiles, offers practitioner-verifiable metrics, validated on synthetic and real-world datasets. Future work includes scaling CDP for large point clouds using approximate shortest-path methods and exploring applications in dynamic environments like robotics.

\end{document}